\documentclass[accepted]{uai2026} 
                        
\usepackage[american]{babel}

\usepackage{natbib} 
\usepackage{booktabs} 
\usepackage{tikz} 
\usepackage{amsmath}
\usepackage{algorithm}
\usepackage{algorithmic}

\usepackage{mathtools}
\usepackage{amsthm}
\usepackage{bm}
\usepackage{subcaption}
\usepackage{amsfonts}
\newtheorem{proposition}{Proposition}
\newtheorem{lemma}{Lemma}
\newcommand{\rkone}[1]{\textbf{\textcolor{red}{#1}}}
\newcommand{\rktwo}[1]{\textcolor{blue}{#1}}
\newcommand{\rkthree}[1]{\textcolor{teal}{#1}}

\title{Implicit Variational Rejection Sampling}

\author[1,2]{Jian~Xu\thanks{Email: \texttt{jian.xu@riken.jp}}}
\author[3]{Shigui~Li}
\author[3]{Wei~Chen}
\author[3]{Jiacheng~Li}
\author[3]{Zhiqi~Lin}
\author[3]{\protect\\ Delu~Zeng\thanks{Corresponding author. Email: \texttt{dlzeng@scut.edu.cn}}}
\author[4]{Xinghao~Ding}
\author[5]{John~Paisley}
\author[2]{Qibin~Zhao\thanks{Corresponding author. Email: \texttt{qibin.zhao@riken.jp}}}
\affil[1]{RIKEN iTHEMS}
\affil[2]{RIKEN AIP}
\affil[3]{South China University of Technology}
\affil[4]{Xiamen University}
\affil[5]{Columbia University}
  
\begin{document}
\maketitle

\begin{abstract}
Variational Inference (VI) is a fundamental inference technique in Bayesian machine learning for approximating complex posterior distributions. Traditional VI often relies on the mean-field factorization, which can inadequately capture true posterior complexity. Recent advancements have leveraged neural networks to model implicit distributions, offering increased flexibility. However, the practical constraints of neural network architectures still produces inaccuracies. In this paper, we propose a method called Implicit Variational Rejection Sampling (IVRS), which integrates implicit distributions with rejection sampling to improve the posterior approximation. Our method uses neural networks to construct implicit proposal distributions, and rejection sampling with a discriminator network that estimates the density ratio between the implicit proposal and the true posterior for refining the approximation. Towards this end, we introduce the Implicit Resampling Evidence Lower Bound (IR-ELBO) as a metric to characterize the resampled distribution's quality and derive a tighter variational lower bound. Experimental results demonstrate that our method outperforms traditional variational inference techniques.
\end{abstract}

\section{Introduction} 

Variational Inference (VI) has emerged as a fundamental technique in Bayesian machine learning for approximating complex posterior distributions \citep{jordan1999introduction, hoffman2013stochastic}. Traditional VI methods frequently rely on a mean-field assumption \citep{blei2017variational}, which trades off posterior expressiveness for computational tractability. To address this limitation, implicit distributions, which are typically modeled using neural networks, have been proposed to leverage their flexibility when approximating complex posterior distributions \citep{mescheder2017adversarial, huszar2017variational, titsias2019unbiased, shi2017kernel}; such implicit and diffusion-based constructions have also been extended to richer Bayesian models such as deep Gaussian processes \citep{xu2024sparse, xu2026diffusion}. While neural networks are highly expressive in theory \citep{hornik1989multilayer, krizhevsky2012imagenet, lecun2015deep}, they may still struggle to match complex true posteriors in practice, especially under limited capacity, poor initialization, or optimization difficulty 
 \citep{arora2017generalization}. As a result, the posterior approximations with neural networks are not robust as an off-the-shelf method. 

To improve neural network posterior approximations, we propose Implicit Variational Rejection Sampling (IVRS), which leverages rejection sampling \citep{gilks1992adaptive} to better exploit the strengths of implicit distributions. We first use neural networks to construct implicit distributions that serve as proposals. We then design an acceptance probability function related to the density ratio between the proposal distribution and the true posterior, and apply rejection sampling to generate resampled samples. A discriminator network is used to approximate the density ratio, thereby refining the proposal distribution into a more accurate posterior approximation. By incorporating adversarial training techniques, this approach enables us to construct an Implicitly Resampled Evidence Lower Bound (IR-ELBO). 

We summarize our contributions below:
\begin{enumerate}
    \item[1)] We introduce Implicit Variational Rejection Sampling (IVRS) to combine implicit distributions with rejection sampling for more accurate variational inference with neural networks.
    \item[2)] By incorporating adversarial training techniques, we construct an Implicit Resampling Evidence Lower Bound (IR-ELBO) and analyze the resampled distribution, particularly its reduced KL divergence from the true posterior, providing theoretical support for improved accuracy.
    \item[3)] We demonstrate through experiments that IVRS can outperform traditional variational inference methods in terms of accuracy and efficiency.
\end{enumerate}

\section{Model Framework}

\subsection{Bayesian Generative Models}

Consider an unsupervised generative model for a dataset \(\mathcal{D} = \{\mathbf{x}_i\}_{i=1}^N\) that has latent variables \(\mathbf{z}\) and model parameters \(\mathbf{\theta}\). The joint distribution of the model has the form
\begin{equation}
\label{1}
p(\mathbf{x}, \mathbf{z} | \theta) = p(\mathbf{z}) p(\mathbf{x}|\mathbf{z}, \theta),    
\end{equation}
where \(p(\mathbf{z})\) is the prior distribution of \(\mathbf{z}\) and \(p(\mathbf{x}|\mathbf{z}, \theta)\) is a parametric generative model. In more traditional models, such as the Gaussian Mixture Model (GMM), this distribution is often specified through manual design. With the advance of deep learning, and generative models in particular as exemplified by Variational Autoencoders (VAEs), this distribution is frequently parameterized using a neural network. While this framework can be extended to supervised learning scenarios, we develop our method in the unsupervised learning regime.


\subsection{Variational Inference}
The goal of inference is to model the posterior distribution of the latent variables in Equation \ref{1}. For non-conjugate models, such as those involving deep learning architectures, the posterior distribution is highly complex and requires approximate methods. Variational inference (VI) provides a KL divergence approximation to the posterior distribution \(p(\mathbf{z}|\mathbf{x})\) using a predefined variational distribution \(q(\mathbf{z}| \phi)\) by maximizing the Evidence Lower Bound (ELBO),
\begin{equation}
\label{elbo}
\mathcal{L}(\mathbf{x}, \theta, \phi) = \mathbb{E}_{q(\mathbf{z}| \phi)} \left[ \log p(\mathbf{x}, \mathbf{z}| \theta) - \log q(\mathbf{z}| \phi) \right].
\end{equation}
The traditional mean-field approximation assumes a factorized form for the variational distribution,
\begin{equation}
q(\mathbf{z}|\mathbf{x}, \phi) = \textstyle\prod_{i=1}^m q(\mathbf{z}_i|\mathbf{x}, \phi_i),
\end{equation}
where \(m\) represents the number of factors in the decomposition, and each \(q(\mathbf{z}_i|\mathbf{x}, \phi_i)\) is often a simple parametric distribution.

The mean field approximation sacrifices accuracy for tractability. To address this limitation, the implicit variational inference method employs a neural network parameterization of the variational distribution. These methods aim to capture more accurate, complex posterior structures by leveraging the expressive power of neural networks. They take the form
\begin{equation}
\label{4}
\mathbf{z}\sim q_{\phi}(\mathbf{z}|\mathbf{x})\quad \longrightarrow\quad  \mathbf{z} = f_{\phi}\left( \mathbf{x},\epsilon \right), ~~\epsilon \sim p(\epsilon ).
\end{equation}
Here, \(\phi\) represents the parameters of a neural network, while \(\epsilon\) is drawn independently from a simple distribution such as a Gaussian. Recently, various algorithms have been proposed to effectively train such models, including Adversarial Variational Bayes \citep{mescheder2017adversarial} and Semi-implicit Variational Inference \citep{yin2018semi}. 

Despite their flexibility, traditional neural networks face practical limitations, and their structural design often relies on empirical heuristics. Rejection sampling \citep{naesseth2017reparameterization, jankowiak2023reparameterized} offers a flexible approach to more robust implicit distribution learning without requiring additional model capacity or architectural changes. We therefore view rejection sampling as a complementary mechanism for improving implicit VI, particularly when the variational family lacks sufficient support. In the following section, we introduce a method that incorporates rejection sampling to address these challenges.

\section{Proposed Method}
\subsection{Rejection Sampling}
Rejection sampling is a standard statistical technique for generating samples from a target distribution via a proposal distribution. Given a target distribution \(p_\mathrm{tar}(\mathbf{z})\) and a proposal distribution \(q_\mathrm{pro}(\mathbf{z})\), rejection sampling accepts a sample \(\mathbf{z} \sim q_\mathrm{pro}(\mathbf{z})\) with probability defined by an acceptance probability function $a(\mathbf{z})$ that is proportional to the ratio of the target density to the proposal density,
\begin{equation}
\label{5}
a(\mathbf{z}) = p_\mathrm{tar}(\mathbf{z}) / Mq_\mathrm{pro}(\mathbf{z}),
\end{equation}
where \(M > 0\) and the choice of \(M\) ensures that the acceptance probability is less than or equal to 1. In our model, the target distribution \(p_\mathrm{tar}(\mathbf{z}) = p_\theta(\mathbf{z}|\mathbf{x})\), being the true posterior of the latent variables of model structure Equation \ref{1}. We define the proposal distribution to be the implicit distribution in Equation \ref{5}, \(q_\mathrm{pro}(\mathbf{z}) = q_\phi(\mathbf{z}|\mathbf{x})\). Therefore, we write the acceptance rate function as \(a(\mathbf{z}; \mathbf{x}, \theta, \phi)\). 

However, unlike traditional rejection sampling, calculating the acceptance rate function \(a(\mathbf{z}; \mathbf{x}, \theta, \phi)\) in our model is not analytical. The primary challenges are two-fold: i) The target distribution \(p_\theta(\mathbf{z}|\mathbf{x})\) is the true posterior, typically only representable in the unnormalized joint likelihood form of Bayes rule in Equation \ref{1}; ii) The proposal distribution \(q_\phi(\mathbf{z}|\mathbf{x})\) is an implicit distribution, often generated by  a structure that makes it difficult to determine its probability density function. We next turn to our proposal for addressing these two issues.

\subsection{The Acceptance Probability Function}

To address these two challenges, we first express the target distribution \(p_\theta(\mathbf{z}|\mathbf{x})\) using Bayes rule,
\begin{equation}
p_\theta(\mathbf{z}|\mathbf{x}) = p_\theta(\mathbf{x}|\mathbf{z})p(\mathbf{z}) / p_\theta(\mathbf{x}),
\end{equation}
where \(p_\theta(\mathbf{x}|\mathbf{z})\) is the likelihood, \(p(\mathbf{z})\) is the prior, and \(p_\theta(\mathbf{x})= \int p_\theta(\mathbf{x}|\mathbf{z}) p(\mathbf{z}) \, d\mathbf{z}\) is the evidence. To ensure the acceptance probability is within the range \([0, 1]\), we can ignore the evidence term provided we choose an appropriate scaling factor \(M\) such that
\begin{equation}
\label{7}
a(\mathbf{z}; \mathbf{x}, \theta, \phi) = \frac{p_\theta(\mathbf{x}|\mathbf{z})p(\mathbf{z})}{Mq_\phi(\mathbf{z}|\mathbf{x})} \leq 1.
\end{equation}
To ensure this constraint in practice,  the acceptance rate function is usually  constructed as
\begin{equation}
\label{8}
a(\mathbf{z}; \mathbf{x}, \theta, \phi) = \min \left[ \frac{p_\theta(\mathbf{x}|\mathbf{z}) p(\mathbf{z})}{M q_\phi(\mathbf{z}|\mathbf{x})}, 1 \right].
\end{equation}
However, the min function makes gradient-based optimization for variational posterior parameters challenging. To address this, we adopt the fully differentiable approximation of \cite{grover2018variational},
\begin{equation}
\label{101}
a(\mathbf{z}; \mathbf{x}, \theta, \phi)  = \frac{p_{\theta}(\mathbf{x}|\mathbf{z}) p(\mathbf{z})}{p_{\theta}(\mathbf{x}|\mathbf{z}) p(\mathbf{z}) + M q_{\phi}(\mathbf{z}|\mathbf{x})} \in (0, 1)
\end{equation}
This approximation addresses the first challenge.

For the second challenge, where the proposal distribution \(q_\phi(\mathbf{z}|\mathbf{x})\) is implicitly modeled by a neural network \(\phi\), we estimate it using adversarial training. Specifically, we address the challenge of computing the term \(\log p(\mathbf{z}) - \log q_\phi(\mathbf{z}|\mathbf{x})\) by introducing an additional discriminative network \(T(\mathbf{x}, \mathbf{z})\), which distinguishes between pairs \((\mathbf{x}, \mathbf{z})\) sampled from the true joint distribution \(p(\mathbf{x}, \mathbf{z})\), and pairs \((\mathbf{x}, \mathbf{z})\) sampled using the implicit proposal distribution \(q_\phi(\mathbf{z}|\mathbf{x})\). The objective $D(T)$ for this discriminator \(T(\mathbf{x}, \mathbf{z})\) is
\begin{equation}
\label{9}
\begin{aligned}
D(T)~=~&\mathbb{E}_{p(\mathbf{x})} \mathbb{E}_{q_\phi(\mathbf{z}|\mathbf{x})} \left[ \log \sigma(T(\mathbf{x}, \mathbf{z})) \right] \\+~& \mathbb{E}_{p(\mathbf{x})} \mathbb{E}_{p(\mathbf{z})} \left[ \log (1 - \sigma(T(\mathbf{x}, \mathbf{z}))) \right],  
\end{aligned}
\end{equation}
where \(\sigma(t) = \frac{1}{1 + e^{-t}}\) denotes the sigmoid function. By \cite{goodfellow2014generative} and \cite{mescheder2017adversarial}, the optimal discriminator \(T^*(\mathbf{x},\mathbf{z})\) is
\begin{equation}
\label{10}
T^*(\mathbf{x},\mathbf{z})=\log q_{\phi}(\mathbf{z}|\mathbf{x})-\log p(\mathbf{z}).    
\end{equation}
We see that \(T^*\) can be directly substituted into Equation (\ref{101}) to compute the implicit proposal distribution,
\begin{equation}
\label{11}
a(\mathbf{z};\mathbf{x},\theta ,\phi )=\frac{p_{\theta}(\mathbf{x}|\mathbf{z})}{p_{\theta}(\mathbf{x}|\mathbf{z})+M\exp \left( T^*(\mathbf{x},\mathbf{z}) \right)}.
\end{equation}
As a result, we can effectively perform rejection sampling even in the absence of explicit analytical forms for the proposal distribution.  
Numerous methods are available for estimating density ratios of non-analytical distributions. We employ adversarial training here \citep{goodfellow2014generative, mescheder2017adversarial}, although alternative estimators, such as recent diffusion- and Schr{\"o}dinger-bridge-based density-ratio estimation \citep{chen2025dequantified}, are equally compatible with our framework. Additionally, the expectation in Equation (\ref{9}) is on the outermost layer, so Monte Carlo estimation remains unbiased and is suitable for mini-batch algorithms.

\begin{algorithm}[t]

\caption{Sampler for \( r_{\theta, \phi}(\mathbf{z}|\mathbf{x}) \)}
\begin{algorithmic}[1]
\label{alg1}
\REQUIRE \( a_{\theta, \phi}(\mathbf{z}; \theta, \phi) \), \( q_\phi(\mathbf{z}|\mathbf{x}) \)
\ENSURE \( \mathbf{z} \sim r_{\theta, \phi}(\mathbf{z}|\mathbf{x}) \)
\STATE Perform gradient ascent on \( D(T_\eta) \) in Equation (\ref{9}) with respect to \( \eta \) to obtain \( T_\eta^* \)

\WHILE{True}
    \STATE \( \mathbf{z} \leftarrow \) sample from implict proposal \( q_\phi(\mathbf{z}|\mathbf{x}) \) by Equation (\ref{4})
    \STATE Compute acceptance probability \( a(\mathbf{z}; \mathbf{x}, \theta, \phi) \) by Equation (\ref{11})
    \STATE Sample uniform \( u \sim \mathcal{U}[0, 1] \)
    \IF{ \( u < a(\mathbf{z}; \mathbf{x}, \theta, \phi) \) }
        \STATE Output sample \( \mathbf{z} \)
    \ENDIF
\ENDWHILE
\end{algorithmic}
\end{algorithm}

\subsection{Implicit Resampling ELBO}
Unlike previous implicit variational inference, we us as our variational approximation the distribution resampled via rejection sampling, denoted as
\begin{equation}
r_{\theta, \phi}(\mathbf{z}|\mathbf{x}) = \frac{q_\phi(\mathbf{z}|\mathbf{x}) a(\mathbf{z}; \mathbf{x}, \theta, \phi)}{Z_{\theta, \phi}(\mathbf{x})},
\end{equation}
where
$Z_{\theta, \phi}(\mathbf{x}) = \mathbb{E}_{q_\phi(\mathbf{z}|\mathbf{x})}[a(\mathbf{z}; \mathbf{x}, \theta, \phi)].$

To sample from \(r_{\theta, \phi}\), we follow the procedure defined in Algorithm \ref{alg1}. First, we use a neural network parameterized by \(\eta\) to represent the discriminative network \(T_\eta(\mathbf{x}, \mathbf{z})\). By using gradient-based optimization, we obtain a local optimal value \(T_\eta^*(\mathbf{x}, \mathbf{z})\). Then, through an  accept-reject step, we resample from the implicit proposal distribution \(r_{\theta, \phi}\).  
We then define the \textit{Implicit Resampling Evidence Lower Bound} (IR-ELBO) on the marginal log-likelihood of $\mathbf{x}$. This involves using the implicit distribution as the proposal distribution and the resampled distribution \(r_{\theta, \phi}\) as the variational distribution. By Jensen’s inequality, we have that,
\begin{equation}
\begin{aligned}
\label{150}
\log p_{\theta}(\mathbf{x})&\ge \mathbb{E} _{r_{\theta ,\phi}(\mathbf{z}|\mathbf{x})}\left[ \log \frac{p_{\theta}(\mathbf{x},\mathbf{z})}{r_{\theta, \phi}(\mathbf{z}|\mathbf{x})} \right] \\&=\mathbb{E} _{r_{\theta ,\phi}(\mathbf{z}|\mathbf{x})}\left[ \log \frac{p_{\theta}(\mathbf{x},\mathbf{z})Z_{\theta, \phi}(\mathbf{x})}{q_{\phi}(\mathbf{z}| \mathbf{x})a(\mathbf{z};\mathbf{x}, \theta, \phi )} \right]. 
\end{aligned}
\end{equation}
In this equation, we can use the discriminative network \( T_\eta^*(\mathbf{x}, \mathbf{z}) \) described in Algorithm \ref{alg1} to compute the probability density function of the implicit distribution. Using Equations (\ref{101}) and (\ref{10}), we therefore have that
\begin{align}\label{160}
&\mathbb{E} _{r_{\theta ,\phi}}\left[ \log \frac{p_{\theta}(\mathbf{x},\mathbf{z})Z_{\theta ,\phi}(\mathbf{x})}{q_{\phi}(\mathbf{z}|\mathbf{x})a(\mathbf{z};\mathbf{x},\theta ,\phi )} \right]~= \\
&~~~ \mathbb{E} _{r_{\theta ,\phi}}\bigg[ \log \bigg( \frac{p_{\theta}(\mathbf{x}|\mathbf{z})}{\exp \left( T_{\eta}^{*}\left( \mathbf{x},\mathbf{z} \right) \right)}+M \bigg) \bigg] +\log  Z_{\theta ,\phi}(\mathbf{x}).\nonumber
\end{align}
For the term \(\log Z_{\theta ,\phi}(\mathbf{x})\), we can again use Jensen's inequality to define a lower bound,
\begin{align}\label{14}
\log Z_{\theta ,\phi}(\mathbf{x})&\ge \mathbb{E} _{q_{\phi}(\mathbf{z}|\mathbf{x})}[\log a(\mathbf{z};\theta ,\phi )]\\
&=\mathbb{E} _{q_{\phi}}\bigg[ \log \frac{p_{\theta}\left( \mathbf{x}|\mathbf{z} \right)}{p_{\theta}\left( \mathbf{x}|\mathbf{z} \right) +M\exp \left( T_{\eta}^{*}\left( \mathbf{x},\mathbf{z} \right) \right)} \bigg].\nonumber
\end{align}
Substituting the lower bound for \(\log Z_{\theta ,\phi}(\mathbf{x})\) from Equation (\ref{14}) into Equation (\ref{160}) yields the final loss function, which we call the IR-ELBO. Similar to Equation (\ref{9}), since the expectation is applied to the outermost layer, the Monte Carlo approximation of this objective function remains unbiased and is appropriate for mini-batch algorithms. Samples from the implicit distribution can be directly obtained, and by adjusting the parameter \(M\) they can be resampled.

\subsection{Implicit Variational Rejection Sampling}
Using the above derivations, we propose a new inference method for generative models called Implicit Variational Rejection Sampling (IVRS), which combines the strengths of implicit distributions and rejection sampling to achieve a more accurate posterior approximation. The algorithm is summarized in Algorithm \ref{alg2}. Optimization of the discriminator network \(T_\eta(\mathbf{x}, \mathbf{z})\) is reflected in both Algorithm \ref{alg1} and Algorithm \ref{alg2} --- these steps can be merged in practice.

\begin{algorithm}[t]
\begin{algorithmic}[1]
\caption{Implicit Variational Rejection Sampling}

\label{alg2}
\REQUIRE Data $\mathbf{x}$, model parameters $\theta$, neural network parameters $\phi$ and $\eta$
\ENSURE Optimized parameters $\theta^*$ and $\phi^*$

\STATE \textbf{Sample Generation:} Generate samples $\{\mathbf{z}_i\}_{i=1}^Q$ from implicit proposal $q_\phi(\mathbf{z}|\mathbf{x})$.

\STATE \textbf{Density Ratio:} Use discriminator network $T_\eta(\mathbf{x}, \mathbf{z})$ to estimate density ratio for each $\mathbf{z}_i$.

\STATE \textbf{Rejection Sampling:} Accept/reject $\mathbf{z}_i$ based on acceptance function $a(\mathbf{z}_i; \mathbf{x}_i,\theta, \phi)$. (Alg. \ref{alg1})

\STATE \textbf{ELBO Optimization:} Compute the IR-ELBO by Equation (\ref{160}) and Equation (\ref{14}).
\STATE Update $\theta \leftarrow \theta + \alpha \nabla_\theta \text{IR-ELBO}$.
\STATE Update $\phi \leftarrow \phi + \beta \nabla_\phi \text{IR-ELBO}$.

\STATE Repeat steps 1 to 6 until convergence.

\end{algorithmic}
\end{algorithm}

\paragraph{Discussion of IVRS.}
We briefly analyze the properties of the resampling distribution \( r_{\theta, \phi}(\mathbf{z}|\mathbf{x}) \) and show that it is indeed a better approximation compared to the implicit proposal distribution \( q_\phi(\mathbf{z}|\mathbf{x}) \). To that end, we directly compute the KL divergence between \( r_{\theta, \phi}(\mathbf{z}|\mathbf{x}) \) and the true posterior \( p_{\theta}(\mathbf{z}|\mathbf{x}) \),
\begin{align}\label{15}
&\mathrm{KL}\left( r_{\theta ,\phi}(\mathbf{z}|\mathbf{x})||p_{\theta}(\mathbf{z}|\mathbf{x}) \right) ~= ~~~\\&~~~~~\int{\frac{q_{\phi}(\mathbf{z}|\mathbf{x})a(\mathbf{z};\mathbf{x}, \theta ,\phi )}{Z_{\theta ,\phi}(\mathbf{x})}\log \frac{q_{\phi}(\mathbf{z}|\mathbf{x})a(\mathbf{z};\mathbf{x}, \theta, \phi )}{Z_{\theta ,\phi}(\mathbf{x})p_{\theta}(\mathbf{z}|\mathbf{x})}}d\mathbf{z}.\nonumber
\end{align}
We directly substitute Equation (\ref{101}) into the above KL divergence. The RHS can then be rewritten as
\begin{equation}
\label{16}
\int{\frac{p_{\theta}(\mathbf{z}|\mathbf{x})q_{\phi}(\mathbf{z}|\mathbf{x})}{\Lambda}\log \frac{q_{\phi}(\mathbf{z}|\mathbf{x})}{\Lambda}}d\mathbf{z},
\end{equation}
where $\Lambda=Z_{\theta ,\phi}(\mathbf{x})\left( p_{\theta}(\mathbf{z}|\mathbf{x})+Mq_{\phi}(\mathbf{z}|\mathbf{x}) \right)$.

We now characterize how this divergence depends on the scaling parameter \(M\).
\begin{proposition}\label{prop:mono}
Consider the idealized setting in which the discriminator is optimal, so that the acceptance probability in Equation~(\ref{101}) uses the exact density ratio. Then \( \mathrm{KL}\left( r_{\theta ,\phi}(\mathbf{z}|\mathbf{x}) \parallel p_{\theta}(\mathbf{z}|\mathbf{x}) \right) \) is a monotonically non-increasing function of \(M\). Moreover, as \(M\rightarrow 0\) the acceptance rate approaches \(1\) and \( r_{\theta ,\phi} \) approaches \( q_\phi(\mathbf{z}|\mathbf{x})\), whereas as \(M \rightarrow \infty\) the acceptance rate approaches \(0\) and \( r_{\theta ,\phi} \) approaches the true posterior \( p_{\theta}(\mathbf{z}|\mathbf{x}) \).
\end{proposition}
The key step is to write the acceptance probability in the reduced form \(a_M(\mathbf{z})=1/(1+M\,w(\mathbf{z}))\), with density ratio \(w(\mathbf{z})=q_\phi(\mathbf{z}|\mathbf{x})/p_\theta(\mathbf{z}|\mathbf{x})\). Differentiating the divergence then gives \(\tfrac{d}{dM}\mathrm{KL}(r_{\theta,\phi}\parallel p_\theta)=-\mathrm{Cov}_{r_{\theta,\phi}}\!\big(c_M,\log c_M\big)\le 0\), where \(c_M=w/(1+Mw)\) and the inequality follows because \(\mathrm{Cov}(X,f(X))\ge 0\) for any monotone non-decreasing \(f\). A complete derivation is provided in Appendix~\ref{app:mono}. It follows that, in our algorithm, $M$ contributes positively to model improvement by driving $r_{\theta, \phi}$ to provide a more accurate approximation compared to the implicit proposal distribution $q_\phi(\mathbf{z}|\mathbf{x})$, thereby achieving a tighter variational lower bound. In our experiments we empirically determine \(M\) by cross-validation. When the discriminator is only approximately optimal, the acceptance probability and the induced resampled distribution become approximate, so the monotonicity above is best understood as an idealized analytical property rather than a strict guarantee.

\section{Experiments}
We compare IVRS with other ELBO-based variational inference methods, including the Adversarial Variational Bayes (AVB) \citep{mescheder2017adversarial} and Semi-Implicit Variational Inference (SIVI) \citep{yin2018semi}, across a range of unsupervised learning tasks. These include some illustrative studies on several toy examples, followed by a comparison one a regression tast with Bayesian Neural Networks (BNNs) and an autoencoder task. We also consider variants of SIVI, such as UIVI \citep{titsias2019unbiased}. 
We use the Adam optimizer \citep{kingma2014adam} and empirically select \( M \) for training via cross-validation. All experiments were conducted on an RTX 4090. The code will be available with a final draft of the paper.

\begin{figure*}[t!]
    \centering
    \begin{subfigure}[b]{0.245\textwidth}
        \centering
        \includegraphics[trim={20mm 12mm 15mm 5mm},clip,width=\textwidth]{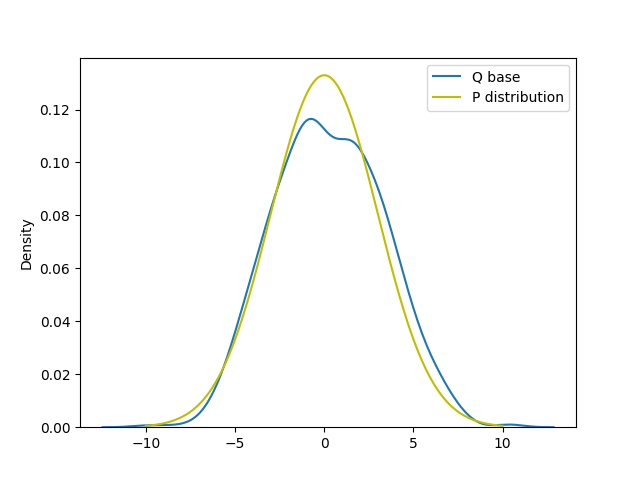}
        \caption{1D Gaussian }
    \end{subfigure} 
    \begin{subfigure}[b]{0.245\textwidth}
        \centering
        \includegraphics[trim={20mm 12mm 15mm 5mm},clip,width=\textwidth]{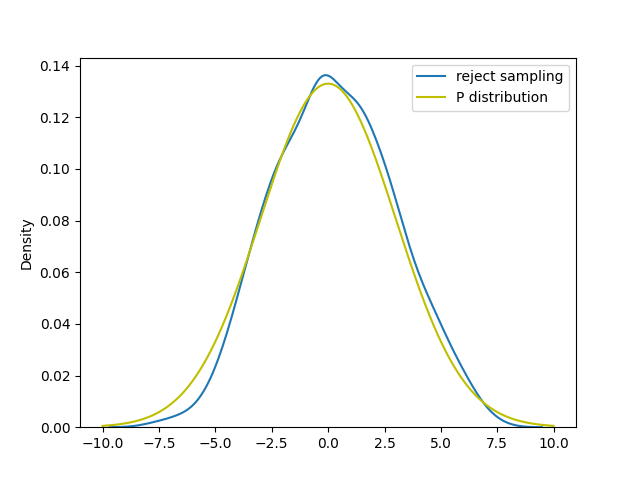}
        \caption{Rejection sampling}
    \end{subfigure} 
    \begin{subfigure}[b]{0.245\textwidth}
        \centering
        \includegraphics[trim={20mm 12mm 15mm 5mm},clip,width=\textwidth]{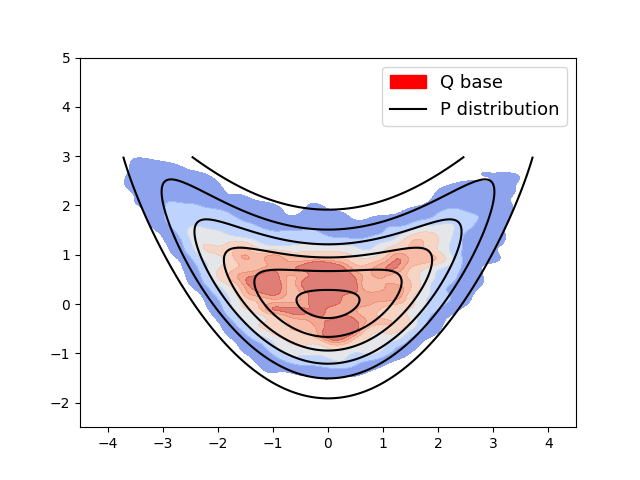}
        \caption{Banana shape}
    \end{subfigure} 
    \begin{subfigure}[b]{0.245\textwidth}
        \centering
        \includegraphics[trim={20mm 12mm 15mm 5mm},clip,width=\textwidth]{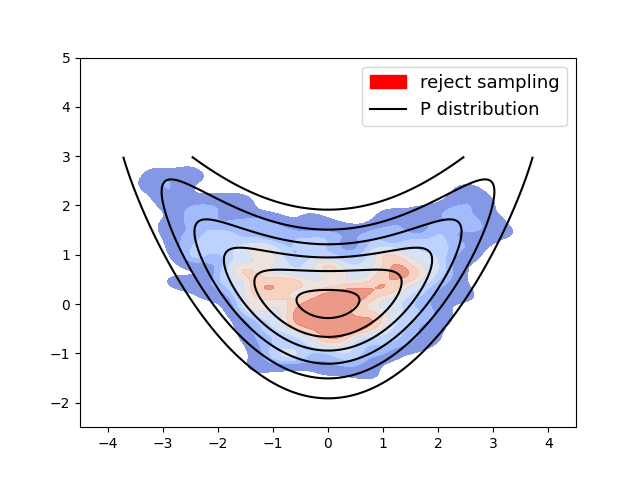}
        \caption{Rejection sampling}
    \end{subfigure} \\
    
    \begin{subfigure}[b]{0.245\textwidth}
        \centering
        \includegraphics[trim={20mm 12mm 15mm 5mm},clip,width=\textwidth]{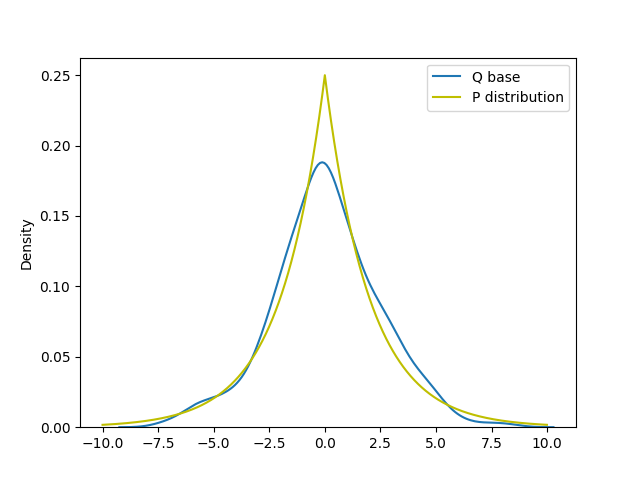}
        \caption{1D Laplace}
    \end{subfigure} 
    \begin{subfigure}[b]{0.245\textwidth}
        \centering
        \includegraphics[trim={20mm 12mm 15mm 5mm},clip,width=\textwidth]{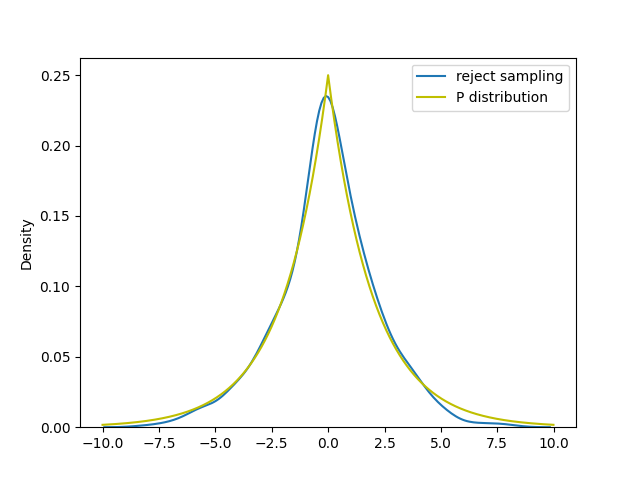}
        \caption{Rejection sampling}
    \end{subfigure} 
    \begin{subfigure}[b]{0.245\textwidth}
        \centering
        \includegraphics[trim={20mm 12mm 15mm 5mm},clip,width=\textwidth]{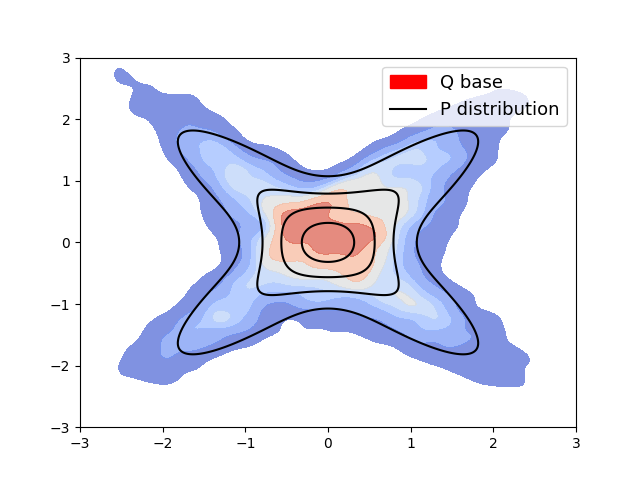}
        \caption{X-shape}
    \end{subfigure} 
    \begin{subfigure}[b]{0.245\textwidth}
        \centering
        \includegraphics[trim={20mm 12mm 15mm 5mm},clip,width=\textwidth]{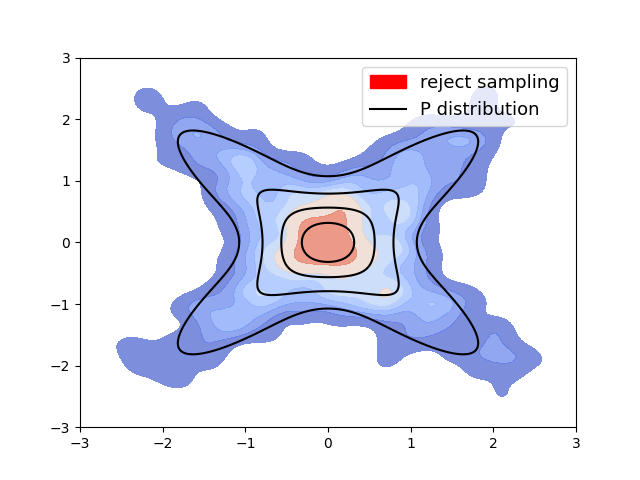}
        \caption{Rejection sampling}
    \end{subfigure} \\
    
    \begin{subfigure}[b]{0.245\textwidth}
        \centering
        \includegraphics[trim={20mm 12mm 15mm 5mm},clip,width=\textwidth]{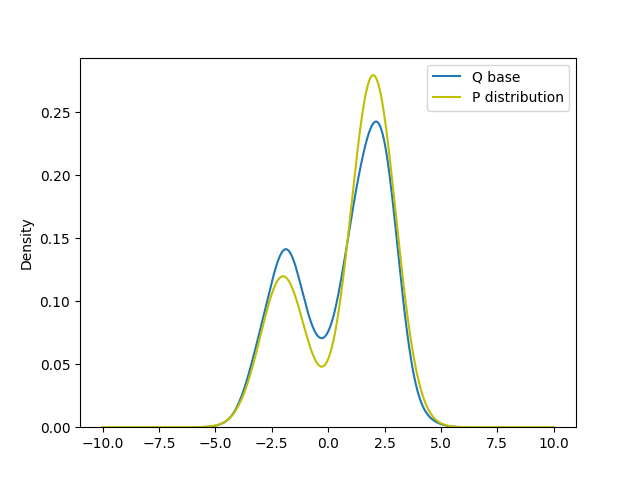}
        \caption{1D GMM}
    \end{subfigure} 
    \begin{subfigure}[b]{0.245\textwidth}
        \centering
        \includegraphics[trim={20mm 12mm 15mm 5mm},clip,width=\textwidth]{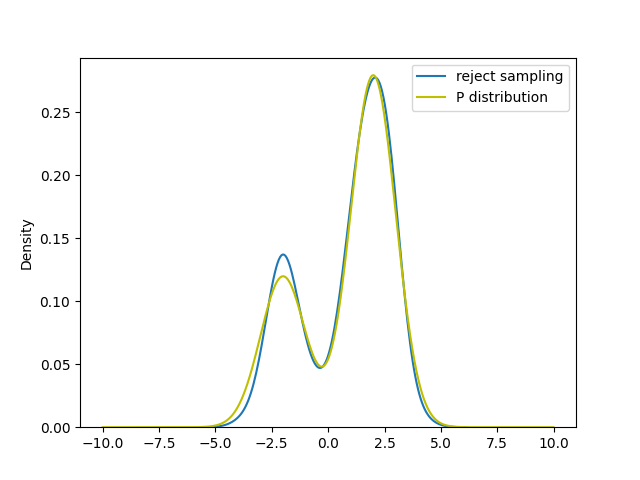}
        \caption{Rejection sampling}
    \end{subfigure} 
    \begin{subfigure}[b]{0.245\textwidth}
        \centering
        \includegraphics[trim={20mm 12mm 15mm 5mm},clip,width=\textwidth]{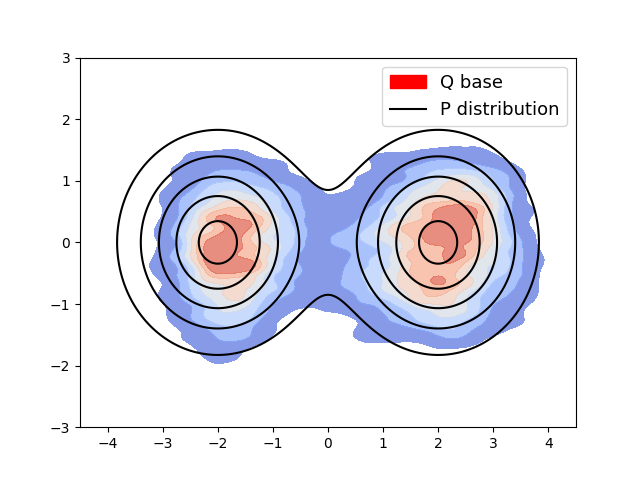}
        \caption{2D GMM}
    \end{subfigure} 
    \begin{subfigure}[b]{0.245\textwidth}
        \centering
        \includegraphics[trim={20mm 12mm 15mm 5mm},clip,width=\textwidth]{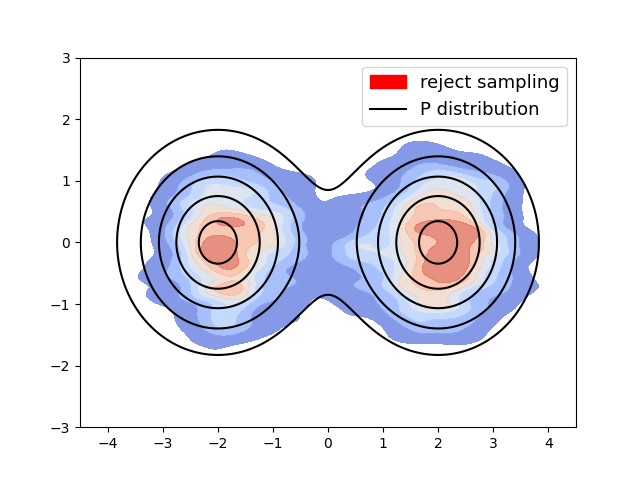}
        \caption{Rejection sampling}
    \end{subfigure}
    \caption{Density estimation tasks for 1D and 2D toy datasets. The subplots sequentially show the kernel density curves estimated using the original implicit distribution $Q$ base for the target distribution \(P\), as well as the curves after applying rejection sampling. The color gradient represents the density of the generated samples, with red indicating higher concentration and blue indicating sparser regions. Table \ref{tab:1} shows quantitative values.}
    \label{fig:toy}
\end{figure*}

\begin{table}[th!]
\renewcommand{\arraystretch}{1.1} 
\resizebox{1\columnwidth}{!}{\begin{tabular}{|c|c|c|c|}
\hline
Distribution  & $-\mathbb{E}_{q_{\phi}}[\log p_\mathrm{tar}]$ & w/ reject samp &$M$ \\
\hline
Gaussian  &$-0.095\small{\pm0.001} $ &$-0.102\small{\pm0.001} $&$0.1$  \\\hline
Laplace  & $-0.107\small{\pm0.001} $ & $-0.112\small{\pm0.001} $&$0.1$ \\\hline
GMM-1D &$-0.101\small{\pm0.000} $  &$-0.104\small{\pm0.000} $&$0.1$  \\\hline
Banana   &$ ~\,\,1.235\small{\pm0.005} $&$~\,\,1.189\small{\pm0.003} $ &$500$\\\hline
X-shape 
& $~\,\,0.744\small{\pm0.008}$ & $~\,\,0.682\small{\pm0.005}$ &$500$\\\hline
GMM-2D  & $~\,\,1.292\small{\pm0.006} $ &  $~\,\,1.214\small{\pm0.004}$&$500$\\
\hline
\end{tabular}}
\caption{A comparison of cross-entropy values before and after rejection sampling for different targets using implicit distributions as proposal distributions. We observe that rejection sampling provides improved approximate posterior samples. (Lower is better.)}
\label{tab:1}
\end{table}


\subsection{Density Estimation of Toy Datasets}

We IVRS to approximate six low dimensional synthetic distributions: A 1D Gaussian distribution, a 1D Laplace distribution, a 1D bimodal Gaussian Mixture Model (GMM), a 2D banana-shaped distribution, a 2D X-shaped mixture of Gaussians, and a 2D bimodal GMM. The six densities are pictured in Figure \ref{fig:toy} and listed in Table \ref{tab:1} along with the performance results.

The dimension of \(\epsilon\) was set to 10, and the network approximation \(f_\phi\) was parameterized by a 4-layer MLP with layer widths $[20, 40, 20, 2]$. The output of \(f_\phi\) was then combined with Gaussian noise. Since this task is straightforward, we adopt a Monte Carlo estimator to estimate \(q_\phi\) and  \(\mathrm{KL}(q_\phi || p_{\text{tar}})\) for gradient-based optimization. The key difference in our method is that the trained neural network was used as the proposal distribution, followed by rejection sampling using the acceptance probability function defined in Equation \ref{101}. Following \cite{yin2018semi}, we used 100 iterations for each inner-loop of Monte Carlo sampling to estimate the entropy of the implicit distribution. All methods were trained with 50,000 parameter updates.

\begin{table*}[tp!]

\centering
\resizebox{1\textwidth}{!}{
\begin{tabular}{c|cc|cc|cc} 
\hline\hline
& \multicolumn{2}{c|}{\textsc{Boston}} & \multicolumn{2}{c|}{\textsc{Concrete}} & \multicolumn{2}{c}{\textsc{Protein}}\\

& NLL ($\downarrow$) & RMSE ($\downarrow$) & NLL ($\downarrow$) & RMSE ($\downarrow$) & NLL ($\downarrow$) & RMSE ($\downarrow$) \\ 

\hline

AVB & ${2.489 \pm 0.02}$ & ${2.685 \pm 0.03}$ & ${3.406 \pm 0.01}$ & ${7.091 \pm 0.03}$ & ${2.969 \pm 0.05}$ & ${4.670 \pm 0.04}$\\ 
SIVI & $2.481\pm0.00$ & $2.621\pm0.02$ & $3.337 \pm 0.00$ & $6.932 \pm 0.02$ & $2.967 \pm 0.03$ & $4.669 \pm 0.02$ \\ 
UIVI & $2.490\pm0.02$ & $2.617\pm0.03$ & $3.331 \pm 0.01$ & $6.806 \pm 0.02$ & $2.973 \pm 0.03$ & $4.671 \pm 0.02$ \\ 
SIVI-SM & $2.542\pm0.01$ & $2.785\pm0.03$ & $3.229 \pm 0.01$ & $5.973\pm 0.04$ & $3.047 \pm 0.00$ & $5.087 \pm 0.01$\\ 
KSIVI & $2.506\pm0.01$ & $2.555\pm0.02$ & $3.309\pm 0.01$ & $5.750 \pm 0.03$ & $3.034 \pm 0.04$ & $5.027 \pm 0.01$\\ 
\hline
  IVRS & $\bm{2.365 \pm 0.03}$ & $\bm{2.421 \pm 0.03}$ & $\bm{2.964 \pm 0.01}$ & $\bm{5.68 \pm 0.04}$ & $\bm{2.794 \pm 0.04}$ & $\bm{4.601 \pm 0.03}$\\ 
\hline\hline
\end{tabular}}
\vspace{12pt}

\resizebox{1\textwidth}{!}{
\begin{tabular}{c|cc|cc|cc} 
\hline\hline
& \multicolumn{2}{c|}{\textsc{Power}} & \multicolumn{2}{c|}{\textsc{Wine}} & \multicolumn{2}{c}{\textsc{Yacht}}\\

& NLL ($\downarrow$) & RMSE ($\downarrow$) & NLL ($\downarrow$) & RMSE ($\downarrow$) & NLL ($\downarrow$) & RMSE ($\downarrow$)\\ 

\hline

AVB & ${2.795 \pm 0.02}$ & $3.865 \pm 0.02$ & ${0.905 \pm 0.01}$ & ${0.609 \pm 0.00}$ & ${1.751 \pm 0.06}$ & ${1.567 \pm 0.04}$\\ 
SIVI & $2.791\pm0.00$ & $3.861\pm0.01$ & $0.904 \pm 0.00$ & $0.597 \pm 0.00$ & $1.721 \pm 0.03$ & $1.505 \pm 0.07$ \\ 
UIVI & $2.794\pm0.00$ & $3.863\pm0.02$ & $0.907 \pm 0.00$ & $0.613 \pm 0.00$ & $1.808 \pm 0.03$ & $1.569 \pm 0.05$ \\ 
SIVI-SM & $2.822\pm0.00$ & $4.009\pm0.00$ & $0.916 \pm 0.00$ & $0.615 \pm 0.00$ & $1.432 \pm 0.01$ & $\bm{0.884 \pm 0.01}$\\ 
KSIVI & $2.797\pm0.00$ & $3.868 \pm 0.01$ & $0.901 \pm 0.00$ & $0.595 \pm 0.00$ & $1.752 \pm 0.03$ & $1.237 \pm 0.05$\\ 
\hline
IVRS & \bm{$2.670 \pm 0.01$} & \bm{$3.684 \pm 0.03$} & $\bm{0.900 \pm 0.00}$ & $\bm{0.591 \pm 0.00}$ & $\bm{1.421 \pm 0.03}$ & $1.065 \pm 0.04$\\ 
\hline\hline
\end{tabular}
}\caption{Quantitative results for six UCI regression tasks. More accurate posterior sampling allows IVRS to outperform several other VI approximations for learning the same Bayesian neural network.}\label{tab:reg2results}
\end{table*}

Figure \ref{fig:toy} shows the contour plots of the synthetic distributions, along with kernel density estimates from samples drawn from the trained implicit distributions \(q_\phi\). Additionally, Figure \ref{fig:toy} compares the distributions before and after applying rejection sampling. The results show the neural network's slight misalignment with the target distribution contour. However, after applying rejection sampling, our method demonstrates improved approximation with better alignment to the target distribution as expected. Due to the challenges in normalizing the distribution after rejection sampling, we instead report the cross-entropy between the target distributions and the approximate distributions in Table \ref{tab:1}. We conducted 10 runs and report the mean and standard deviation of the cross-entropy values. As shown, our method outperforms across all toy target distributions by reducing the cross-entropy, further validating its effectiveness.

\subsection{Bayesian Neural Networks}
We next consider the problem of sampling from the posterior of a Bayesian neural network (BNN). In this scenario, the latent model variables \(\mathbf{z}\) correspond to the BNN weights. We utilize a two-layer network with 50 hidden units and ReLU activation functions. We compare our method with AVB, SIVI, and several SIVI variants, including UIVI, SIVI-SM \citep{yu2023semi}, and KSIVI \citep{cheng2024kernel}. We use six common UCI datasets to perform these experiments. Each dataset is randomly partitioned, with 90\% used for training and 10\% for testing. Both the proposal distribution \(\phi\) and the discriminator \(\eta\) are modeled using four-layer fully connected neural networks. The results are averaged over 10 random trials. 

Table~\ref{tab:reg2results} presents the average test root mean squared error (RMSE) and negative log-likelihood (NLL) along with their standard deviations. The six UCI datasets considered are indicated by their names. As is evident, IVRS achieves competitive or superior performance compared to the baselines on nearly all problems, indicating that rejection sampling provides a more accurate representation of the BNN model variables. Additional ablations---fixed-architecture comparisons, discriminator-optimization sensitivity, latent-dimensionality scaling, and a decomposition of the gains---are provided in Appendix~\ref{sec:ablation}.

\subsection{Sensitivity to the Rejection Hyperparameter $M$}
The hyperparameter $M$ controls the acceptance probability and directly shapes the resampled variational distribution $r_{\theta,\phi}(\mathbf{z}\mid\mathbf{x})$: increasing $M$ tightens the approximation by pushing $r_{\theta,\phi}$ toward the true posterior, at the cost of a lower acceptance rate and higher computational overhead. We evaluate IVRS with $M\in\{0.1,1,10,100,500\}$ on the three UCI BNN benchmarks (\textit{Boston}, \textit{Concrete}, \textit{Protein}), keeping all other settings fixed and averaging over $10$ runs. As shown in Table~\ref{tab:ablation_M}, larger $M$ consistently lowers both NLL and RMSE---empirically confirming the monotonicity established in Proposition~\ref{prop:mono}---while the acceptance rate decreases and the relative cost grows only moderately (within $2$--$3\times$ even at $M=500$). The main results in Table~\ref{tab:reg2results} correspond to $M=100$; moderate values $M\in[10,100]$ already capture most of the gains at a reasonable cost, which is why we select $M$ in this range via lightweight cross-validation.

\begin{table*}[t]
\centering
\caption{Impact of the rejection hyperparameter $M$ on performance and efficiency. Relative cost is computed as the inverse of the acceptance rate. Lower NLL and RMSE indicate better performance.}
\label{tab:ablation_M}
\small
\begin{tabular}{l c c c c c}
\toprule
Dataset & $M$ & NLL ($\downarrow$) & RMSE ($\downarrow$) & Accept Rate & Rel. Cost \\
\midrule
Boston & 0.1   & 2.478 & 2.648 & 92.3\% & 1.08$\times$ \\
       & 1     & 2.412 & 2.502 & 83.7\% & 1.19$\times$ \\
       & 10    & 2.385 & 2.451 & 71.2\% & 1.40$\times$ \\
       & 100   & 2.365 & 2.421 & 54.8\% & 1.82$\times$ \\
       & 500   & 2.361 & 2.418 & 41.5\% & 2.41$\times$ \\
\midrule
Concrete & 0.1 & 3.285 & 6.845 & 88.5\% & 1.13$\times$ \\
         & 1   & 3.142 & 6.254 & 76.8\% & 1.30$\times$ \\
         & 10  & 3.021 & 5.892 & 63.4\% & 1.58$\times$ \\
         & 100 & 2.964 & 5.680 & 48.2\% & 2.07$\times$ \\
         & 500 & 2.951 & 5.652 & 35.6\% & 2.81$\times$ \\
\midrule
Protein & 0.1  & 2.928 & 4.658 & 94.1\% & 1.06$\times$ \\
        & 1    & 2.857 & 4.632 & 85.2\% & 1.17$\times$ \\
        & 10   & 2.812 & 4.614 & 73.8\% & 1.35$\times$ \\
        & 100  & 2.794 & 4.601 & 61.5\% & 1.63$\times$ \\
        & 500  & 2.789 & 4.598 & 48.7\% & 2.05$\times$ \\
\bottomrule
\end{tabular}
\end{table*}

\subsection{Variational Autoencoder Task on MNIST}

Variational Autoencoders (VAEs) \citep{kingma2013auto} are a popular method for unsupervised feature learning and dimensionality reduction that learn encoder \(\phi\) and decoder \(\theta\) parameters by maximizing the Evidence Lower Bound (ELBO) as defined in Equation \ref{elbo}. Typical VAE implementations utilize Gaussian distributions and amortized inference to approximate a complex posterior distribution \(q_\phi\). To improve this approximation various approaches such as Adversarial Variational Bayes (AVB) and Semi-Implicit Variational Inference (SIVI) have been proposed, leveraging adversarial training and semi-implicit hierarchical structures, respectively. 

To evaluate our proposed Implicit Variational Rejection Sampling (IVRS) on the VAE inference problem, we conducted experiments on the MNIST dataset and compare with AVB and other methods. We trained our model on 60,000 training samples and evaluated performance on 10,000 test samples. The encoder is designed as a two-layer convolutional neural network (CNN) to map to the latent space, while the decoder consists of a four-layer transposed convolution network for image reconstruction. Figure \ref{fig2} shows examples sampled from the test set and our model, demonstrating that our method generates images with a closer resemblance to the ground truth compared to the baseline AVB method, owing to the improved sampling quality introduced by rejection sampling.
\begin{figure}[tp!]
    \centering
    \begin{subfigure}[b]{1\columnwidth}
        \centering
        \includegraphics[trim={22.5mm 5mm 15mm 5mm},clip,width=1\textwidth]{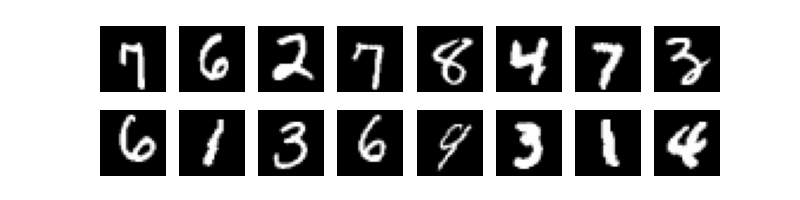}
        \caption{ Ground Truth}
    \end{subfigure}
    \begin{subfigure}[b]{1\columnwidth}
        \centering
        \includegraphics[trim={22.5mm 5mm 15mm 5mm},clip,width=1\textwidth]{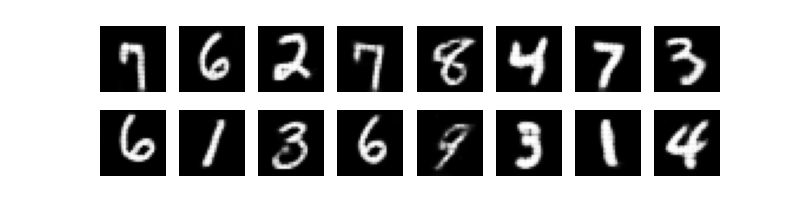}
        \caption{ AVB}
    \end{subfigure}
    \begin{subfigure}[b]{1\columnwidth}
        \centering
        \includegraphics[trim={22.5mm 5mm 15mm 5mm},clip,width=1\textwidth]{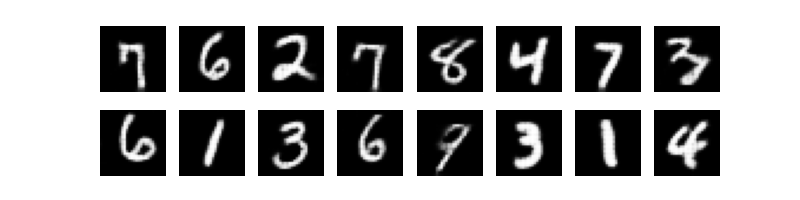}
        \caption{ IVRS (ours)}
    \end{subfigure}

    \caption{Examples of 16 MNIST images from the test set using the VAE learned by AVB and IVRS. Our method generates images with a closer resemblance to the ground truth compared to the baseline due to the rejection sampling.}
    \label{fig2}
\end{figure}

\begin{table}[tp!]

\centering
\resizebox{1\columnwidth}{!}{\begin{tabular}{|c|c|c|c|c|}

\hline
Method & batch & iter & per epoch & per iter\\
\hline
AVB & 64 & 938 & 19.11s& 0.020s\\
\hline
IVRS (Ours) & 64 & 938 & 24.72s&0.026s \\
\hline
\end{tabular}}
\caption{An empirical time analysis comparing the AVB method and our proposed IVRS on the MNIST dataset. GPU-accelerated parallelization is able to reduce the impact of additional computation required by our method.}
\label{time}
\end{table}

To evaluate the efficiency of our model, we conducted an empirical time analysis comparing IVRS to the baseline AVB on MNIST. The results are shown in Table \ref{time}. Using a batch size of 64 during training, we observed that our model incurs only a slight increase in computational cost, showing the advantage of our GPU-accelerated parallel sampling implementation.

\begin{table}[t!]
\setlength{\tabcolsep}{4pt} 
\renewcommand{\arraystretch}{1.1} 
\centering
\resizebox{0.8\columnwidth}{!}{
\begin{tabular}{|c|c|}
\hline
\textbf{Methods} & \textbf{NLE} \\
\hline
\multicolumn{2}{|c|}{\textbf{Results from \cite{burda2015importance}} }  \\
\hline
\begin{tabular}{l}
VAE + IWAE \\
IWAE + IWAE
\end{tabular} & 
\begin{tabular}{l}
$86.76$ \\
$84.78$
\end{tabular} \\
\hline
\multicolumn{2}{|c|}{\textbf{Results from \cite{salimans2015markov}}} \\
\hline
\begin{tabular}{l}
VAE + HVI (1 leapfrog step) \\
VAE + HVI (4 leapfrog steps) \\
VAE + HVI (8 leapfrog steps)
\end{tabular} & 
\begin{tabular}{l}
$88.08$ \\
$86.40$ \\
$85.51$
\end{tabular} \\
\hline
\multicolumn{2}{|c|}{\textbf{Results from \cite{rezende2015variational}}} \\
\hline
\begin{tabular}{l}
VAE + NICE \citep{dinh2014nice}  $(\mathrm{k}=80)$ \\
VAE + NF $(\mathrm{k}=40)$ \\
VAE + NF $(\mathrm{k}=80)$
\end{tabular} & 
\begin{tabular}{l}
$ 87.2$ \\
$ 85.7$ \\
$ 85.1$
\end{tabular} \\
\hline
\multicolumn{2}{|c|}{\textbf{Results from \cite{gregor2015draw}}} \\
\hline
\begin{tabular}{l}
NADE \\
DBM 2hl \\
DBN 2hl \\
EoNADE-5 2hl (128 orderings) \\
DARN 1hl
\end{tabular} & 
\begin{tabular}{l}

$88.33$ \\
$ 84.62$ \\
$ 84.55$ \\
$84.68 $\\
$ 84.13$
\end{tabular} \\
\hline
\multicolumn{2}{|c|}{\textbf{Results from \cite{sonderby2016ladder}}} \\
\hline
Auxiliary VAE ($L=1$, IW $=1$) & $ 84.59$ \\
\hline
\multicolumn{2}{|c|}{\textbf{Results from \cite{mescheder2017adversarial}}} \\
\hline
\begin{tabular}{l}
VAE + IAF \cite{kingma2016improved} \\
AVB 
\end{tabular} & 
\begin{tabular}{l}
$ 84.9$ \\
$ 83.7$
\end{tabular} \\
\hline
\multicolumn{2}{|c|}{\textbf{Results from \cite{yin2018semi}}} \\
\hline
SIVI (3 stochastic layers) + IW($\tilde{K}=10$) & $83.25$ \\
\hline
\multicolumn{2}{|c|}{\textbf{Results from \cite{hazami2022efficientvdvae}}} \\
\hline
\begin{tabular}{l}

PixelVAE++ \citep{sadeghi2019pixelvae++}\\
Local Mask PixelCNN \citep{jain2020locally}\\
NVAE \citep{vahdat2020nvae} \\
CR-NVAE \citep{sinha2021consistency} \\
Efficient-VDVAE \citep{hazami2022efficientvdvae} \\

\end{tabular} & 
\begin{tabular}{l}
$78.00$\\
$77.58$\\
$78.01$ \\
${\rkone{76.93}}^\star$ \\
$79.09$ \\

\end{tabular} \\
\hline
\multicolumn{2}{|c|}{\textbf{Results from \cite{kuzina2024hierarchical}}} \\
\hline
DVP-VAE \citep{kuzina2024hierarchical} & $\rktwo{77.10}$ \\
\hline
\textbf{IVRS}& $81.78$\\

\textbf{IVRS}+NVAE \citep{vahdat2020nvae} & $\rkthree{77.36}$\\
\hline
\end{tabular}
}

\caption{Comparison of reported Negative Log Evidence (NLE) values across different algorithms on the MNIST dataset. An asterisk ($\star$) indicates results obtained with data augmentation. The \rkone{best}, \rktwo{second-best}, and \rkthree{third-best} results are highlighted.}
\label{tab:comparison}
\end{table}

We also quantitatively benchmarked our model against several well-established autoencoding methods, including Importance Weighted Autoencoders (IWAE) \citep{burda2015importance}, Hamiltonian Variational Inference (HVI) \citep{salimans2015markov}, and Normalizing Flows (NF) \citep{rezende2015variational}. Many approaches have also adopted deep architectures for the VAE encoder to achieve more effective feature extraction. Therefore, we also report comparisons with recent deep architecture-based VAE improvements such as NVAE \citep{vahdat2020nvae}, CR-NVAE \citep{sinha2021consistency}, and Efficient-VDVAE \citep{hazami2022efficientvdvae}.

As shown in Table \ref{tab:comparison}, our vanilla IVRS method achieves a Negative Log Evidence (NLE) score of \(81.78\), surpassing traditional variational inference methods like SIVI and AVB. Additionally, we considered integrating IVRS within the NVAE architecture, which further boosts performance to an NLE of \(77.36\). This also represents an improvement over the original NVAE at \(78.01\). Although CR-NVAE \citep{sinha2021consistency} is the one approach to demonstrate slightly better results, that approach relies on additional data augmentation techniques. While data augmentation is highly effective in preventing overfitting, our focus is on demonstrating the competitiveness of rejection sampling in enhancing implicit variational inference, but we include those results for completeness. We further compare against the more recent DVP-VAE \citep{kuzina2024hierarchical}, a hierarchical VAE with a diffusion-based VampPrior, which reports an NLE of \(77.10\); this is directly comparable to our IVRS+NVAE result of \(77.36\), while IVRS additionally offers a general inference-refinement mechanism that is orthogonal to the choice of prior.

\subsection{Results on CIFAR-10  and ImageNet}
To further evaluate IVRS on higher-dimensional image data, we conducted VAE learning experiments on both the CIFAR-10 and ImageNet datasets and compare against baseline inference methods.

The CIFAR-10 dataset consists of 50,000 training images and 10,000 test images, each consisting of \(32 \times 32\) pixels and 3 color channels (RGB) across 10 classes. The ImageNet dataset \citep{deng2009imagenet}, a large-scale dataset commonly used for image classification and generative modeling tasks, includes over 14 million images spanning 1,000 object categories. For our experiments, we utilized a subset of ImageNet consisting of \(64 \times 64\) pixel color images, which are more manageable for generative models like VAEs. This subset introduces greater inference complexity due to the diversity of object categories and higher resolution compared to CIFAR-10. Given the increased complexity and higher dimensionality of these datasets compared to MNIST, we adapted the VAE architecture by directly employing the encoder from the deep architecture NVAE \citep{vahdat2020nvae}.

For quantitative evaluation, we calculate bits per dimension (bpd), a standard metric for high-dimensional image datasets. Table~\ref{tab1:comparison} reports the bpd scores for various variational inference methods on CIFAR-10, while Table~\ref{tabi:comparison} presents the corresponding results on ImageNet. As can be seen, the proposed IVRS consistently achieves lower bpd values than a range of classic VAE-based approaches, indicating improved density estimation quality. Notably, while CR-NVAE~\citep{sinha2021consistency} attains strong performance by leveraging additional data augmentation, our controlled experiments under the same augmentation protocol show that incorporating IVRS further improves performance. In particular, on CIFAR-10, CR-NVAE achieves 2.51~bpd with augmentation, and adding IVRS (with $M=1$) reduces this to 2.43~bpd. This demonstrates that IVRS provides a non-trivial gain even on top of strong augmentation-based baselines, suggesting that the proposed inference refinement is orthogonal and complementary to data augmentation. For a more recent point of comparison, the DVP-VAE \citep{kuzina2024hierarchical} reports $2.73$~bpd on CIFAR-10, which is on par with our IVRS+NVAE result of $2.76$~bpd, while IVRS remains complementary as an inference-refinement step. Moreover, despite the additional rejection sampling step, the overall computational overhead remains modest, indicating that IVRS is practical for large-scale VAE training scenarios.

\begin{table}[t]
\renewcommand{\arraystretch}{1.1} 
\centering
\resizebox{1\columnwidth}{!}{
\begin{tabular}{|c|c|}
\hline
\textbf{Methods} & \textbf{bpd}~~~\\
\hline
\begin{tabular}{l}
VAE + IAF \citep{sonderby2016ladder}\\
BIVA \citep{maaloe2019biva} \\
DVAE \citep{vahdat2018dvae++} \\
$\delta$-VAE \citep{razavi2019preventing} \\
PixelVAE++ \citep{sadeghi2019pixelvae++} \\
Local Mask PixelCNN \citep{jain2020locally}\\
MAE \citep{ma2019mae} \\
NVAE \citep{vahdat2020nvae} \\
VDVAE \citep{child2020very}\\
Efficient-VDVAE \citep{hazami2022efficientvdvae}\\
DVP-VAE \citep{kuzina2024hierarchical} \\
CR-NVAE \citep{sinha2021consistency} \\
\end{tabular} &
\begin{tabular}{l}
$3.11$\\
$3.08$\\
$3.38$\\
$2.83$ \\
$2.90$ \\
$2.89$\\
$2.95$ \\
$2.91 $\\
$2.87$\\
$ 2.87$\\
$\rkthree{2.73}$\\
${\rktwo{2.51}}^\star$\\
\end{tabular} \\

\hline
\textbf{IVRS}+NVAE& $2.76~\,$\\
\textbf{IVRS}+CR-NVAE& ${\rkone{2.43}}^\star$\\
\hline
\end{tabular}
}
\caption{Comparison of reported Bits per Dimension (bpd) values among various algorithms for the CIFAR-10 dataset. An asterisk ($\star$) indicates results obtained using data augmentation. The \rkone{best}, \rktwo{second-best}, and \rkthree{third-best} results are highlighted.}
\label{tab1:comparison}
\end{table}

\begin{table}[!t]
\renewcommand{\arraystretch}{1.1} 
\centering
\resizebox{\columnwidth}{!}{
\begin{tabular}{|c|c|}
\hline
\textbf{Methods} & \textbf{bpd}~~~\\
\hline
\begin{tabular}{l}

SPN \citep{menick2018generating} \\
MaCow \citep{ma2019macow} \\

Sparse Transformer \citep{child2019generating} \\
Local Mask PixelCNN \citep{jain2020locally}\\

NVAE \citep{vahdat2020nvae} \\
VDVAE \citep{child2020very}\\

CR-NVAE \citep{sinha2021consistency} \\
\end{tabular} & 
\begin{tabular}{l}

$3.52$\\

$3.69$ \\
$\rkthree{3.44}$\\
$3.64$ \\
$3.58 $\\
$3.52$\\

${\rkone{3.30}}^\star$\\
\end{tabular} \\

\hline
\textbf{IVRS}+NVAE& $\rktwo{3.38}~\,$\\
\hline
\end{tabular}
}
\caption{Comparison of reported Bits per Dimension (bpd) values among various algorithms for the Imagenet 64 $\times$64 dataset. An asterisk ($\star$) indicates results obtained using data augmentation. The \rkone{best}, \rktwo{second-best}, and \rkthree{third-best} results are highlighted.}
\label{tabi:comparison}
\end{table}

\section{Conclusion}
We have introduced Implicit Variational Rejection Sampling, a novel posterior approximation method that enhances variational inference by integrating the flexibility of implicit distributions with the precision of rejection sampling. By optimizing the derived Implicit Resampled Evidence Lower Bound (IR-ELBO), IVRS provides a mechanism for enhancing posterior approximation over traditional VI methods. Our experimental results demonstrate the effectiveness of IVRS across various tasks, including regression problems and VAE-based image modeling, achieving superior quantitative performance over related inference approaches. The supplement below contains more practical details about implementation and limitations, along with an in depth description of related work and ablation study.
\bibliography{uai2026-template}

\clearpage
\appendix
\thispagestyle{empty}

\onecolumn

\section{Proof of the Monotonicity of the KL Divergence in $M$}
\label{app:mono}
Here we give a rigorous proof of Proposition~\ref{prop:mono}: in the idealized setting with an optimal discriminator, the divergence $\mathrm{KL}\!\left(r_{\theta,\phi}(\mathbf{z}|\mathbf{x})\parallel p_\theta(\mathbf{z}|\mathbf{x})\right)$ is monotonically non-increasing in $M$. We first record an elementary covariance inequality.

\begin{lemma}\label{lem:cov}
Let $X$ be a real-valued random variable and let $f$ be a monotone non-decreasing function such that the expectations below exist. Then $\mathrm{Cov}(X, f(X)) \ge 0$.
\end{lemma}
\begin{proof}
Let $X'$ be an independent copy of $X$. Then
\begin{equation*}
\mathrm{Cov}(X, f(X)) = \tfrac{1}{2}\,\mathbb{E}\!\left[(X-X')\big(f(X)-f(X')\big)\right].
\end{equation*}
Since $f$ is monotone non-decreasing, $(X-X')\big(f(X)-f(X')\big)\ge 0$ for every pair $(X,X')$. Taking expectations gives $\mathrm{Cov}(X,f(X))\ge 0$.
\end{proof}

\begin{proof}[Proof of Proposition~\ref{prop:mono}]
Write $\pi(\mathbf{z}) = p_\theta(\mathbf{z}|\mathbf{x}) = p_\theta(\mathbf{x}|\mathbf{z})p(\mathbf{z})/p_\theta(\mathbf{x})$ for the true posterior, and let the acceptance probability in Equation~(\ref{101}) be
\begin{equation*}
a_M(\mathbf{z}) = \frac{p_\theta(\mathbf{x}|\mathbf{z})p(\mathbf{z})}{p_\theta(\mathbf{x}|\mathbf{z})p(\mathbf{z}) + M q_\phi(\mathbf{z}|\mathbf{x})}.
\end{equation*}
Using $p_\theta(\mathbf{x}|\mathbf{z})p(\mathbf{z}) = p_\theta(\mathbf{x})\pi(\mathbf{z})$, this can be rewritten as
\begin{equation*}
a_M(\mathbf{z}) = \frac{p_\theta(\mathbf{x})\pi(\mathbf{z})}{p_\theta(\mathbf{x})\pi(\mathbf{z}) + M q_\phi(\mathbf{z}|\mathbf{x})}.
\end{equation*}
Defining the density ratio and rescaled parameter
\begin{equation*}
w(\mathbf{z}) = \frac{q_\phi(\mathbf{z}|\mathbf{x})}{\pi(\mathbf{z})}, \qquad \tilde M = \frac{M}{p_\theta(\mathbf{x})},
\end{equation*}
gives $a_M(\mathbf{z}) = 1/(1+\tilde M w(\mathbf{z}))$. Since $p_\theta(\mathbf{x})$ is a positive constant independent of $\mathbf{z}$, we may absorb it into the scalar parameter and equivalently write $a_M(\mathbf{z}) = 1/(1+M w(\mathbf{z}))$.

The induced resampled distribution is
\begin{equation*}
r_M(\mathbf{z}) = \frac{q_\phi(\mathbf{z}|\mathbf{x})\, a_M(\mathbf{z})}{Z_M}, \qquad Z_M = \mathbb{E}_{q_\phi}[a_M(\mathbf{z})].
\end{equation*}
Using $q_\phi(\mathbf{z}|\mathbf{x}) = w(\mathbf{z})\pi(\mathbf{z})$, we obtain
\begin{equation*}
r_M(\mathbf{z}) = \frac{\pi(\mathbf{z})\, c_M(\mathbf{z})}{\tilde Z_M}, \quad c_M(\mathbf{z}) = \frac{w(\mathbf{z})}{1+M w(\mathbf{z})}, \quad \tilde Z_M = \mathbb{E}_\pi[c_M(\mathbf{z})].
\end{equation*}
Hence the divergence to the true posterior is
\begin{equation*}
\mathrm{KL}(r_M\parallel\pi) = \frac{1}{\tilde Z_M}\mathbb{E}_\pi[c_M\log c_M] - \log \tilde Z_M.
\end{equation*}
Differentiating with respect to $M$ and using
\begin{equation*}
\frac{d c_M}{dM} = -c_M^2, \quad \frac{d \tilde Z_M}{dM} = -\mathbb{E}_\pi[c_M^2], \quad \frac{d (c_M\log c_M)}{dM} = -c_M^2(1+\log c_M),
\end{equation*}
the two terms involving $\mathbb{E}_\pi[c_M^2]/\tilde Z_M$ cancel, leaving
\begin{equation*}
\frac{d}{dM}\mathrm{KL}(r_M\parallel\pi) = -\frac{1}{\tilde Z_M}\mathbb{E}_\pi[c_M^2\log c_M] + \frac{\mathbb{E}_\pi[c_M^2]}{\tilde Z_M^2}\mathbb{E}_\pi[c_M\log c_M].
\end{equation*}
Rewriting the expectations under $r_M$ via $\mathbb{E}_{r_M}[g] = \mathbb{E}_\pi[c_M g]/\tilde Z_M$ yields
\begin{equation*}
\frac{d}{dM}\mathrm{KL}(r_M\parallel\pi) = -\mathbb{E}_{r_M}[c_M\log c_M] + \mathbb{E}_{r_M}[c_M]\,\mathbb{E}_{r_M}[\log c_M] = -\mathrm{Cov}_{r_M}(c_M, \log c_M).
\end{equation*}
Applying Lemma~\ref{lem:cov} with $X = c_M$ and the monotone non-decreasing function $f(X) = \log X$ gives $\mathrm{Cov}_{r_M}(c_M,\log c_M)\ge 0$. Therefore
\begin{equation*}
\frac{d}{dM}\mathrm{KL}\!\left(r_M\parallel p_\theta(\mathbf{z}|\mathbf{x})\right) \le 0,
\end{equation*}
so the divergence is monotonically non-increasing in $M$ in the idealized exact-ratio setting. When the discriminator is imperfect, the acceptance probability and the induced resampled distribution are only approximate, so this monotonicity should be interpreted as an idealized analytical property rather than a strict guarantee.
\end{proof}

\section{Related Work}
The main idea of implicit variational inference is to transform a simple base distribution into a more expressive one using a deep neural network \citep{mescheder2017adversarial, huszar2017variational, titsias2019unbiased, shi2017kernel}. To avoid density ratio estimation, semi-implicit variational inference has been proposed, where the variational distributions are formed through a semi-implicit hierarchical construction, and surrogate ELBOs (asymptotically unbiased) are employed for training \citep{yin2018semi, molchanov2019doubly, moens2021efficient, lim2024particle, yu2023semi, cheng2024kernel}. These methods have made improvements to implicit variational inference or semi-implicit variational inference at various levels. However, these methods do not address the limited expressiveness of neural networks, which may still fall short in certain scenarios. Our work aims to address this issue by proposing an implicit rejection sampling algorithm.

Beyond standard latent-variable models, implicit and diffusion-based variational inference has been increasingly applied to richer Bayesian models, including deep Gaussian processes and their dynamical extensions \citep{xu2024sparse, xu2026diffusion, xu2024neural, xu2025variational, xu2025bayesian, xu2025fully}, Bayesian last-layer models \citep{xu2024flexible}, heavy-tailed Student-$t$ processes \citep{xu2024sparsestudent, xu2026sparse}, and nonlinear probabilistic latent-variable models for industrial sensing \citep{chen2024analyzing}. In parallel, diffusion-based generative models have motivated new tools for posterior sampling and density-ratio estimation \citep{chen2025dequantified, li2026evodiff, chen2024rethinking}, which are complementary to the discriminator-based density-ratio estimation used in IVRS.

On the other hand, rejection sampling is a classical method to generate samples from a distribution using samples drawn from a different distribution \citep{gilks1992adaptive, grover2018variational, azadi2018discriminator, stimper2022resampling, verine2024optimal}. Specifically, \cite{grover2018variational} and \cite{stimper2022resampling} have embedded latent rejection sampling within their training processes, applying it within a variational inference and a normalizing flow framework, respectively. We acknowledge their contributions; however, in this work, we address a different problem—leveraging implicit distributions as proposal distributions for rejection sampling and variational inference, creating a novel variational inference algorithm. 

The prior in \cite{bauer2019resampled} is reformulated as a resampled distribution, whereas our method explicitly derives a new evidence lower bound (IR-ELBO) based on rejection sampling. While both approaches use density ratio estimation, our work focuses on leveraging rejection sampling to construct a tighter variational objective by directly approximating the posterior through implicit distributions. 

Equation (2) in \cite{jankowiak2023reparameterized} and our Equation (13) share a similar mathematical form, but a key difference lies in the nature of the variational distribution: in \cite{jankowiak2023reparameterized}, the distribution is explicit, while in our work, it is implicit. This distinction is significant as it aligns with our objective of improving implicit variational inference by utilizing rejection sampling to create more flexible posterior approximations.

Our method can also be extended to importance sampling \cite{burda2015importance}, and could be adapted for use with mixtures of variational distributions. Using mixtures could enhance the expressivity of the approximate posterior and align well with the results from \cite{hotti2024efficient} and \cite{kviman2023cooperation}. Combining the importance-weighted ELBO (IW-ELBO) \cite{burda2015importance} with Eq. (16) could provide a tighter bound. This would involve integrating importance weighting within the rejection sampling framework, potentially amplifying the advantages of both techniques. This presents an exciting direction for future work.

\section{Ablation Study}
\label{sec:ablation}

In this section, we provide additional ablation experiments that complement the sensitivity analysis of the rejection hyperparameter $M$ in the main text. We focus on three representative UCI Bayesian neural network regression benchmarks: \textit{Boston Housing}, \textit{Concrete}, and \textit{Protein}. These datasets allow us to systematically examine the trade-off between posterior accuracy and computational efficiency induced by rejection sampling.

\subsection{Fixed Architecture Comparison}
\label{sec:ablation_architecture}

To isolate the contribution of rejection sampling from increased network capacity, we conduct controlled experiments where all methods share identical neural network architectures. This allows us to explicitly disentangle the effect of posterior refinement via rejection sampling from that of simply adding more model parameters.

Specifically, we compare the following settings:
\begin{enumerate}
    \item Baseline implicit variational inference (Implicit VI) with a fixed 4-layer MLP.
    \item IVRS applied to the same architecture (no increase in network capacity).
    \item Implicit VI with increased network capacity (2$\times$ and 3$\times$ hidden width).
\end{enumerate}
All models are trained and evaluated under identical optimization settings.


Table~\ref{tab:fixed_architecture} reports the results on three representative UCI BNN regression datasets. Several important observations emerge:

\begin{itemize}
    \item \textbf{Rejection sampling outperforms capacity increase.} Applying IVRS with the same base architecture consistently yields larger improvements in NLL and RMSE than doubling or tripling the network width. For example, on the Boston dataset, IVRS achieves a $4.9\%$ NLL improvement over the baseline, whereas doubling network capacity yields only a $1.1\%$ improvement.
    
    \item \textbf{Capacity alone cannot replicate IVRS gains.} Even with 3$\times$ network capacity, implicit VI underperforms IVRS using the original architecture across all datasets. This indicates that the performance gains of IVRS do not stem from increased representational capacity.
    
    \item \textbf{Complementary mechanisms.} These results suggest that rejection sampling improves posterior approximation in a manner fundamentally different from architectural scaling. IVRS refines the variational distribution through sample-level selection rather than function-level expressiveness.
\end{itemize}

Overall, this ablation demonstrates that IVRS provides a principled and parameter-efficient alternative to increasing network capacity, reinforcing the claim that rejection sampling contributes meaningfully beyond standard neural network design choices.

\begin{table}[t]
\centering
\caption{Ablation study with fixed neural network architecture. All models use a 4-layer MLP; [$k$] denotes hidden width per layer. Lower NLL and RMSE indicate better performance.}
\label{tab:fixed_architecture}
\small
\begin{tabular}{l l c c l}
\toprule
Dataset & Method & NLL ($\downarrow$) & RMSE ($\downarrow$) & Architecture \\
\midrule
Boston
& Implicit VI & 2.489 & 2.685 & 4-layer MLP [50] \\
& + Rejection Sampling (IVRS) & \textbf{2.365} & \textbf{2.421} & Same \\
& + 2$\times$ Network Capacity & 2.461 & 2.638 & 4-layer MLP [100] \\
& + 3$\times$ Network Capacity & 2.445 & 2.612 & 4-layer MLP [150] \\
\midrule
Concrete
& Implicit VI & 3.406 & 7.091 & 4-layer MLP [50] \\
& + Rejection Sampling (IVRS) & \textbf{2.964} & \textbf{5.680} & Same \\
& + 2$\times$ Network Capacity & 3.318 & 6.824 & 4-layer MLP [100] \\
& + 3$\times$ Network Capacity & 3.275 & 6.705 & 4-layer MLP [150] \\
\midrule
Protein
& Implicit VI & 2.969 & 4.670 & 4-layer MLP [50] \\
& + Rejection Sampling (IVRS) & \textbf{2.794} & \textbf{4.601} & Same \\
& + 2$\times$ Network Capacity & 2.928 & 4.645 & 4-layer MLP [100] \\
& + 3$\times$ Network Capacity & 2.905 & 4.628 & 4-layer MLP [150] \\
\bottomrule
\end{tabular}
\end{table}
\subsection{Complete Hyperparameter Settings for $M$}
\label{sec:ablation_M_values}

For completeness and reproducibility, we summarize the values of the rejection hyperparameter $M$ used across all experiments.

\paragraph{Toy Density Estimation (Table~1).}
For low-dimensional toy distributions, we set:
\begin{itemize}
    \item 1D distributions (Gaussian, Laplace, GMM): $M=0.1$
    \item 2D distributions (Banana, X-shape, GMM): $M=500$
\end{itemize}
Higher-dimensional targets require larger $M$ to sufficiently refine the proposal distribution, consistent with the analysis in Section~3.4.

\paragraph{UCI Regression (BNN Experiments, Table~2).}
For Bayesian neural network regression, we selected $M$ via cross-validation on held-out validation sets:
\begin{itemize}
    \item Boston, Concrete, Protein, Power: $M=100$
    \item Yacht, Wine: $M=10$
\end{itemize}
These BNN posteriors are relatively low-dimensional (on the order of hundreds of parameters). At $M=100$, acceptance rates remain practical (approximately $50$--$60\%$), yielding substantial accuracy improvements ($4$--$13\%$ NLL reduction) with moderate computational overhead ($1.6$--$2.1\times$).

\paragraph{VAE Experiments.}
For large-scale VAE training, we fix $M=1$ for all datasets:
\begin{itemize}
    \item MNIST
    \item CIFAR-10
    \item ImageNet 64$\times$64
\end{itemize}
Unlike BNN tasks, VAE training involves significantly larger datasets (e.g., 60K for MNIST, 50K for CIFAR-10, and over 1M samples for ImageNet). To maintain reasonable training time, we use a small $M$, which yields high acceptance rates (approximately $80\%$) while still demonstrating the benefit of rejection sampling. These experiments primarily serve to validate the general applicability of IVRS across different model classes.
\subsection{Training and Inference with Rejection Sampling}
\label{sec:ablation_train_test}

We clarify the role of rejection sampling during both training and inference in IVRS, as this point may be a source of potential confusion.

\paragraph{Training Phase.}
During training, rejection sampling is an integral component of IVRS. The variational distribution is defined as the resampled distribution
$r_{\theta,\phi}(z \mid x)$, and all expectations in the proposed IR-ELBO are taken with respect to this distribution. Consequently, rejection sampling is explicitly applied during training to generate samples from $r_{\theta,\phi}(z \mid x)$, as described in Algorithm~1 and Algorithm~2.

\paragraph{Inference Phase.}
At test time, posterior samples are also drawn from the same resampled variational distribution $r_{\theta,\phi}(z \mid x)$ using the identical rejection sampling procedure. This ensures consistency between training and inference, and allows the improved posterior approximation obtained by IVRS to directly translate into better predictive performance.

\paragraph{Practical Considerations.}
In practice, the acceptance rate is controlled by the hyperparameter $M$ and remains sufficiently high in all reported experiments (see Appendix~\ref{sec:ablation_M_values}). For large-scale settings such as VAE training, we adopt a smaller $M$ to balance computational efficiency and posterior refinement. Importantly, no additional approximation is introduced at test time beyond the rejection sampling mechanism already used during training.

\subsection{Sensitivity to Discriminator Optimization}
\label{sec:ablation_disc}
The quality of the discriminator is a critical part of the method. In our framework the discriminator is not a peripheral component: it is used to approximate the density-ratio quantity that enters the acceptance probability, the induced resampled posterior, and the practical training objective. If the discriminator is inaccurate, the resulting error does not remain local but propagates through the entire procedure. The idealized analysis of Proposition~\ref{prop:mono} should therefore be understood as relying on a sufficiently accurate---in the strongest case optimal---discriminator. In practical optimization, however, the discriminator is only approximately trained, so the acceptance function, the resampled posterior, and the optimized objective are all approximate; we therefore do not claim a general theoretical guarantee that approximation errors in the discriminator cannot degrade the posterior approximation. What is optimized in practice is better viewed as a discriminator-parameterized surrogate objective, which is also why Algorithm~\ref{alg2} updates the discriminator and the variational model in alternation rather than fully solving the discriminator subproblem at each outer step.

Our experiments show that the method is sensitive to discriminator optimization, but not in the trivial sense that stronger discriminator training always helps---overly aggressive updates can degrade both acceptance behavior and downstream posterior quality. On the Concrete dataset, under a fixed evaluation protocol with $100$ accepted posterior samples, we observe the behavior in Table~\ref{tab:disc}: a moderate setting is stable, whereas making the discriminator much stronger (more inner steps) or updating it too frequently drives the acceptance rate toward zero and substantially worsens RMSE/NLL. This indicates that the main failure mode is not underfitting of the discriminator, but \emph{imbalance} in the alternating optimization between the discriminator and the variational model; the sensitivity is primarily a training-dynamics issue rather than a purely approximation-theoretic one.

\begin{table}[h]
\centering
\begin{tabular}{lcc}
\toprule
Discriminator setting & RMSE & NLL \\
\midrule
moderate (\texttt{every=5, steps=2}) & 5.6974 & 3.0267 \\
strong (\texttt{steps=25}) & 11.3274 & 3.9247 \\
frequent (\texttt{every=1}) & 8.8113 & 3.6260 \\
\bottomrule
\end{tabular}
\caption{Effect of discriminator optimization on IVRS (Concrete, $100$ accepted samples). The ``strong'' setting drives the acceptance rate down to approximately $0.0002$.}
\label{tab:disc}
\end{table}

\subsection{Effect of Latent Dimensionality}
\label{sec:ablation_dim}
Although our approach uses a smooth rejection-inspired mechanism rather than classical exact rejection sampling, it still inherits a related scalability issue: as the effective latent dimensionality grows, acceptance behavior can deteriorate and tuning becomes more delicate, so the method is not immune to a curse-of-dimensionality effect. On MNIST, varying the latent dimensionality from $16$ to $392$ under a fixed evaluation protocol gives the results in Table~\ref{tab:dim}: the best NLE is achieved in a moderate range (around $32$--$64$ dimensions), while larger latent spaces gradually reduce the calibrated acceptance rate and worsen NLE. Once the latent space becomes too large, the proposal is harder to refine effectively by rejection-style resampling. Overall, IVRS is most suitable when one has a moderately expressive proposal, manageable latent dimensionality, and sufficient room for posterior refinement through selective resampling.

\begin{table}[h]
\centering
\begin{tabular}{lcccccc}
\toprule
Latent dim & 16 & 32 & 64 & 128 & 256 & 392 \\
\midrule
NLE & 83.65 & 81.48 & 81.58 & 82.10 & 82.72 & 82.83 \\
Accept.\ rate & 0.1563 & 0.1319 & 0.1295 & 0.1255 & 0.1177 & 0.1127 \\
\bottomrule
\end{tabular}
\caption{Effect of latent dimensionality on IVRS (MNIST). ``Accept.\ rate'' is the calibrated acceptance rate.}
\label{tab:dim}
\end{table}

\subsection{Decomposing the Source of the Gains}
\label{sec:ablation_decomp}
To understand where the improvements come from, we perform controlled decomposition ablations on the Concrete dataset using the same backbone, training protocol, and evaluation budget, while isolating four cases: \texttt{baseline} (implicit train + implicit test), \texttt{reweight\_only} (IVRS-style train + implicit test), \texttt{selection\_only} (implicit train + IVRS-style test), and \texttt{full\_ivrs} (IVRS-style train + IVRS-style test). Table~\ref{tab:decomp} reports RMSE/NLL under two discriminator settings ($M{=}1$ and $M{=}0.1$, both with \texttt{disc-update-every=5}, \texttt{disc-steps=1}). The gain cannot be attributed to density-ratio reweighting alone, and ``implicit regularization'' from the training objective by itself is not consistently beneficial; rather, the improvement comes from the \emph{interaction} between reweighting and selection rather than from either component in isolation.

\begin{table}[h]
\centering
\begin{tabular}{lcc}
\toprule
Configuration & $M{=}1$ (RMSE/NLL) & $M{=}0.1$ (RMSE/NLL) \\
\midrule
\texttt{baseline} & 7.068 / 3.4322 & 7.068 / 3.4322 \\
\texttt{reweight\_only} & 7.0230 / 3.4567 & 9.2461 / 3.3719 \\
\texttt{selection\_only} & 9.4048 / 3.5626 & 8.0144 / 3.6206 \\
\texttt{full\_ivrs} & \textbf{6.2992 / 3.1451} & \textbf{6.7390 / 3.2773} \\
\bottomrule
\end{tabular}
\caption{Controlled decomposition ablation on Concrete. Only the full combination of reweighting (train) and selection (test) consistently improves over the baseline.}
\label{tab:decomp}
\end{table}

\section{Limitations and Additional Discussion}
Despite the clear advantages of IVRS, which combines the accuracy of rejection sampling with the flexibility of implicit distributions, there are some limitations. One key limitation is the need for empirical hand-tuning of the hyperparameter $M$ in the acceptance probability function. Additionally, while our experiments have shown improved bpd scores on standard datasets like CIFAR-10, further testing on more complex and higher-dimensional datasets is necessary to fully assess the robustness of our approach.


\subsection{Robustness of the Method in High-Dimensional Settings}
We acknowledge the  challenges in high-dimensional scenarios due to the curse of dimensionality. To mitigate this, our approach leverages implicit proposal distributions parameterized by neural networks, which are designed to closely approximate the target distribution. This reduces the rejection rate and improves efficiency even in higher dimensions.

In our experiments, we have evaluated the proposed method on datasets with moderately high-dimensional posterior distributions (e.g., several hundred dimensions). Results demonstrate that the acceptance rates remain manageable, and the method performs favorably compared to baseline variational inference (VI) approaches. We recognize, however, that testing on more extreme high-dimensional cases (e.g., thousands of dimensions) could provide additional insights. This is a valuable direction for future work to implement more scalable architectures to further explore this aspect.

\subsection{Computational Efficiency Compared to Other Methods}
Computational overhead is a critical factor in assessing the practicality of our approach. The addition of the discriminator network introduces some computational cost, particularly for estimating the density ratio and acceptance probability. To address this issue:

\begin{itemize}
    \item \textbf{Comparison with Baselines:} In our experiments, we observed that the computational cost of training the discriminator is offset by the reduction in bias due to tighter approximation of the Evidence Lower Bound (ELBO). Compared to standard VI methods, our approach exhibits a trade-off where the marginal improvement in accuracy justifies the additional computation.
    \item \textbf{Efficiency in Large-Scale Settings:} To improve scalability, we used a lightweight architecture for the discriminator and optimized its training through mini-batch techniques. However, we acknowledge that in large-scale datasets or extremely high-dimensional problems, further optimization (e.g., parallelization or approximate methods) may be necessary.
    \item \textbf{Maximum Iteration Limit:} We introduce a maximum iteration limit for the rejection sampling process to prevent it from getting stuck in an infinite loop
\end{itemize}

\subsection{Cost and Difficulty of Training the Discriminator (Eq. \ref{9})}

In our framework, the discriminator approximates the density ratio between the proposal distribution $q_\phi(z|x)$ and the true posterior $p(z|x)$, which determines the acceptance probability in the rejection sampling step. It is parameterized using the same architecture as in AVB and trained with a standard binary cross-entropy loss. As a result, the computational cost remains comparable to AVB, and no additional model complexity is introduced.

The discriminator directly influences the sampling process by defining the acceptance probability $a(z; x, \theta, \phi)$, making its accuracy crucial to overall performance. However, we found that full convergence of the discriminator at each update step is not necessary in practice. Instead, we adopt an alternating optimization strategy, where the discriminator is updated for a few epochs (typically 5–10) before updating the variational parameters, similar to the alternating training scheme used in Generative Adversarial Networks (GANs).

\subsection{Manual Tuning of the Hyperparameter $M$}

The hyperparameter $M$ controls the acceptance threshold in rejection sampling and directly governs the trade-off between posterior accuracy and computational efficiency. In our experiments, $M$ is selected via cross-validation on held-out validation sets.

As demonstrated in the ablation studies in Appendix~\ref{sec:ablation}, increasing $M$ consistently improves posterior approximation quality, as reflected by lower NLL and RMSE, while simultaneously reducing the acceptance rate and increasing computational cost. This behavior aligns with the theoretical analysis in Section~3.4, which shows that the KL divergence between the resampled variational distribution and the true posterior decreases monotonically with $M$.

Importantly, we find that IVRS is not overly sensitive to the exact choice of $M$ within a reasonable range. Moderate values of $M$ already provide substantial accuracy gains over the implicit proposal distribution, while maintaining practical acceptance rates. This allows $M$ to be chosen using lightweight cross-validation without extensive tuning.

While automatic or adaptive selection of $M$ is an interesting direction for future work, our results indicate that the current strategy provides a reliable and interpretable balance between accuracy and efficiency across a wide range of tasks.

\end{document}